\documentclass{article}

\usepackage[preprint]{neurips_2026}

\usepackage[utf8]{inputenc} 
\usepackage[T1]{fontenc}    
\usepackage{hyperref}       
\usepackage{url}            
\usepackage{booktabs}       
\usepackage{amsfonts}       
\usepackage{nicefrac}       
\usepackage{microtype}      
\usepackage{xcolor}         

\usepackage{bm}
\usepackage{amsmath}
\usepackage{amsthm}
\usepackage{algorithm}
\usepackage{algorithmic}
\usepackage{multirow}
\usepackage{graphicx}
\usepackage{wrapfig}
\usepackage{arydshln} 
\usepackage{makecell}
\usepackage{subfig}
\newtheorem{proposition}{Proposition}
\usepackage{amssymb}

\title{Spatial Gram Alignment for Ultra-High-Resolution Image Synthesis}

\author{%
  Jinjin Zhang\textsuperscript{1}, Xiefan Guo\textsuperscript{1}, Di Huang\textsuperscript{1} \\
  \textsuperscript{1}Beihang University \\
}

\begin{document}

\maketitle

\begin{abstract}
  Modern ultra-high-resolution image synthesis relies heavily on the robust generative capacity of large-scale pre-trained Latent Diffusion Models (LDMs).
  While recent representation alignment methods have proven effective by distilling visual priors from foundation models (\emph{e.g.}, SAM or DINO) into generative latent features, scaling these approaches to pre-trained LDMs at extreme resolutions exposes a critical \emph{learnability-fidelity conflict}.
  Specifically, forcing direct patch-wise feature distillation inherently perturbs the pre-trained latent manifold, ultimately leading to generation degradation.
  To address this bottleneck, we propose Spatial Gram Alignment (SGA), a novel framework that explicitly leverages the representation priors of vision foundation models while preserving the native generative capacity of LDMs.
  Moving beyond restrictive direct alignment, SGA imposes a non-invasive spatial constraint by aligning the internal self-similarities of the generative features with those of the foundation priors.
  This spatial constraint effectively establishes macroscopic structural coherence, while the native generative objectives retain the microscopic pixel-level fidelity inherent to the original LDMs.
  Notably, this versatile strategy integrates seamlessly across both intermediate diffusion features and VAE latents within pre-trained LDMs.
  Extensive experiments demonstrate that SGA achieves state-of-the-art performance for ultra-high-resolution text-to-image synthesis, yielding an effective reconciliation between global structural integrity and fine-grained visual details. 
  Code is available at \href{https://github.com/zhang0jhon/SGA}{https://github.com/zhang0jhon/SGA}.
\end{abstract}

\section{Introduction}

Latent Diffusion Models (LDMs) have driven remarkable progress in photorealistic high-resolution text-to-image synthesis, as evidenced by milestone architectures such as Imagen~\cite{saharia2022photorealistic, baldridge2024imagen}, DALL·E~\cite{ramesh2022hierarchical, betker2023improving}, Stable Diffusion (SD)~\cite{rombach2022high, esser2024scaling}, and Flux~\cite{Flux:2024:Online}. 
To explicitly leverage deep visual priors for accelerating training convergence and boosting inherent model learnability, recent representation alignment (REPA) approaches have emerged as a powerful paradigm. 
These methods align the generative latent features with the deep semantic spaces of pre-trained vision foundation models, such as DINO~\cite{caron2021emerging, oquab2023dinov2, simeoni2025dinov3} or SAM~\cite{kirillov2023segment, ravi2024sam, carion2025sam}. 
By distilling these visual representation priors into either the intermediate hidden states of diffusion models~\cite{yu2024representation, singh2025matters, leng2025repa} or the latent spaces of Variational AutoEncoders (VAEs)~\cite{yao2025reconstruction, zheng2025diffusion, zhang2025both}, such approaches have demonstrated remarkable effectiveness in enhancing the inherent learnability of generative models.

However, while proven effective when training generative models from scratch on standard benchmarks (\emph{e.g.}, ImageNet~\cite{deng2009imagenet}), extending these REPA approaches to fine-tune large-scale pre-trained LDMs inevitably exposes a critical \emph{learnability-fidelity conflict}, which is particularly amplified in ultra-high-resolution image synthesis. 
As illustrated in Figure~\ref{fig:pca}, PCA visualizations of latent features clearly reveal these distinct properties: vision foundation models encode consistent macroscopic semantic topologies, whereas conventional VAE latents~\cite{kingma2013auto, van2017neural, esser2021taming, esser2024scaling, Flux:2024:Online} preserve microscopic, dense high-frequency information. 
Mechanistically, standard REPA constraints rely on maximizing cross-model patch similarities within projected feature spaces~\cite{yu2024representation, yao2025reconstruction}. 
While these projection heads offer limited mathematical relaxation, this direct cross-model distillation forces the generative latent manifold to homogenize towards the foundation space, thereby compromising the delicate high-frequency variations inherent to the pre-trained LDMs. 
Consequently, the native capacity to synthesize intricate, high-fidelity local details is notably restricted, a limitation that becomes prohibitive at extreme 4K scales. 
As empirically validated in Figure~\ref{fig:degradation}, the direct application of state-of-the-art alignment strategies, such as iREPA~\cite{singh2025matters}, to the Flux model~\cite{Flux:2024:Online} induces noticeable generation degradation, quantitatively evidenced by a marked deterioration in gFID scores~\cite{heusel2017gans} compared to the vanilla fine-tuning baseline. 
These findings explicitly illustrate the \emph{learnability-fidelity conflict}: when subjected to strict feature distillation, LDMs struggle to simultaneously maintain global structural coherence and synthesize fine-grained local details, underscoring the critical need to effectively reconcile the inherent tension between representation learnability and native generative fidelity.

\begin{figure}
  \centering
  \centering 
    \subfloat[Feature Visualizations.]{
    \centering
    \includegraphics[width=0.63\linewidth]{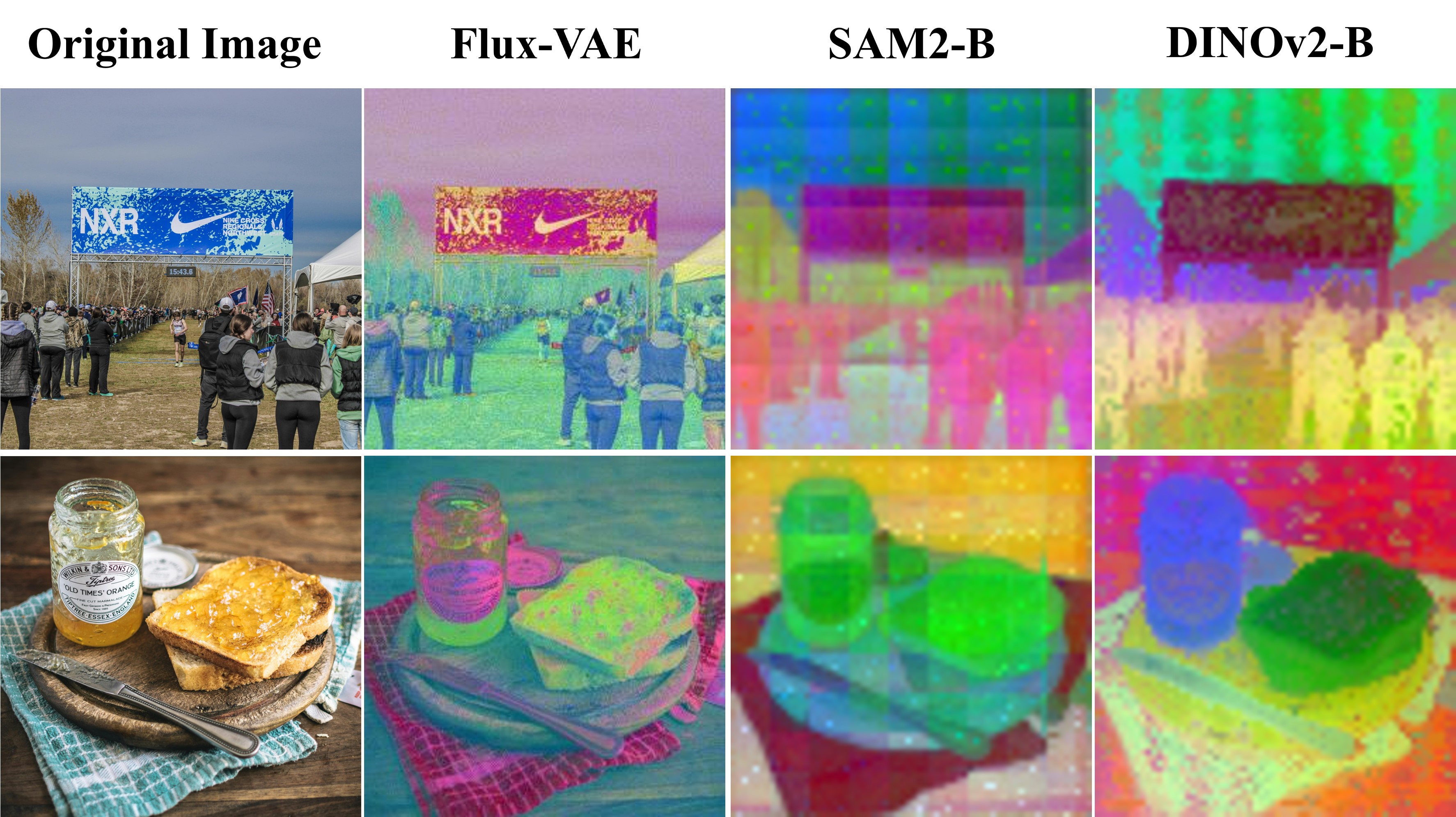}
    \label{fig:pca}
    }
    \hfill
    \subfloat[Fine-tuning w/ and w/o iREPA.]{
    \centering
    \includegraphics[width=0.33\linewidth]{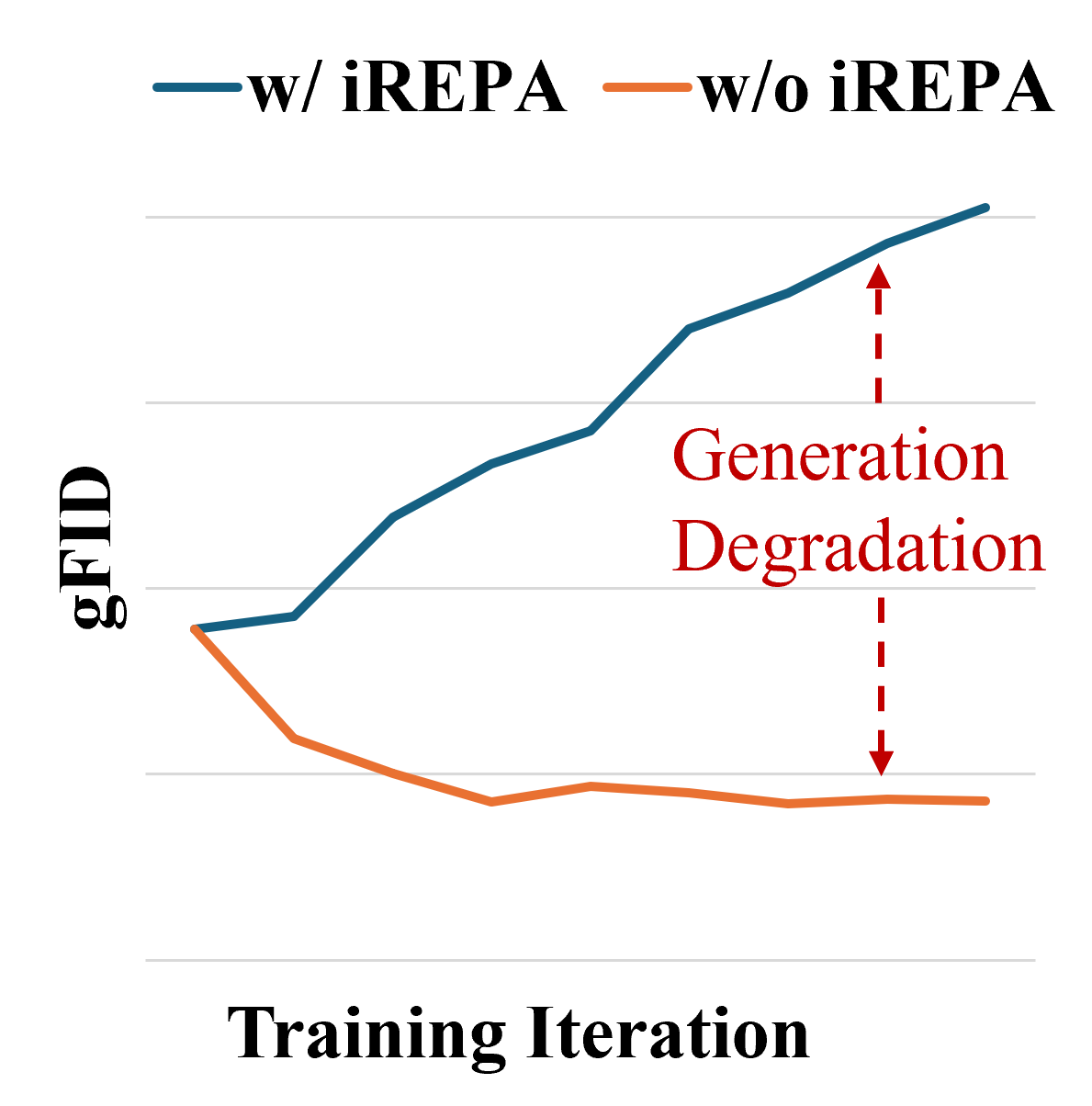}
    \label{fig:degradation}
    }
  \caption{\textbf{Analysis of the \emph{Learnability-Fidelity Conflict}.} (a) PCA feature visualizations reveal distinct representation properties: vision foundation models encode macroscopic semantic topologies, whereas native generative latents preserve microscopic high-frequency details. (b) Directly applying rigid representation alignment (\emph{i.e.}, iREPA~\cite{singh2025matters}) to the Flux model~\cite{Flux:2024:Online} rigidly homogenizes these representation spaces, inducing \emph{generation degradation} at 4K resolution.
  }
  \label{fig:analysis}
\end{figure}

In this paper, we propose Spatial Gram Alignment (SGA), a novel framework designed to combine the distinct advantages of both paradigms, harnessing the representation learnability of foundation models while preserving the fidelity of pre-trained LDMs. 
Instead of enforcing restrictive patch-wise feature distillation that perturbs the pre-trained generative manifold, our approach imposes a non-invasive spatial structural constraint by aligning the internal self-similarities of the generative features with those of the vision foundation models. 
This Gram-based formulation captures essential relative structural relationships while remaining agnostic to the absolute channel basis of the generative features (see Appendix~\ref{appendix:theory}). 
By employing this spatial constraint to establish global structural coherence, while allowing the native generative objectives to safeguard intricate pixel-level fidelity, our framework effectively reconciles representation learnability with uncompromised generative capacity. 
Notably, our method is highly versatile and integrates seamlessly into both VAE latent spaces and intermediate diffusion features, consistently enhancing structural coherence without degrading the pre-trained manifolds of large-scale LDMs. 
Extensive experiments demonstrate the effectiveness of our approach, advancing the current state-of-the-art in ultra-high-resolution image synthesis.

In summary, our main contributions are three-fold:
\begin{itemize}
\item We identify and investigate the critical \emph{learnability-fidelity conflict} when integrating representation priors into large-scale pre-trained LDMs for ultra-high-resolution (\emph{e.g.}, 4K) image synthesis, revealing that strict patch-wise feature homogenization inherently perturbs the native pre-trained manifolds, inevitably inducing generation degradation.
\item We propose SGA, a novel and non-invasive representation alignment framework. Rather than enforcing direct cross-model feature distillation, SGA aligns internal spatial self-similarities to establish macroscopic structural coherence while safeguarding the inherent microscopic generative capacity of LDMs.
\item We demonstrate the versatility of SGA by seamlessly integrating it into both the VAE and the diffusion model. Extensive experiments validate that our approach achieves leading performance for ultra-high-resolution text-to-image synthesis, yielding a superior reconciliation between global representation learnability and fine-grained visual fidelity.
\end{itemize}

\section{Related Work}

\subsection{Ultra-High-Resolution Image Synthesis}

Since training ultra-high-resolution models from scratch remains computationally prohibitive, modern 4K image synthesis pipelines heavily rely on leveraging the rich textual-visual priors of large-scale pre-trained LDMs. 
While recent foundation models in this class, such as SD3/3.5~\cite{esser2024scaling} and Flux~\cite{Flux:2024:Online}, have achieved notable success in standard high-resolution text-to-image synthesis, scaling them to ultra-high-resolution (\emph{e.g.}, 4K) regimes poses significant challenges.
Beyond the substantial quadratic computational overhead, simultaneously synthesizing coherent macroscopic structures and intricate microscopic details remains highly challenging~\cite{zhang2025both, ren2024ultrapixel, zhao2025ultrahr}.

To tackle these extreme scales, various strategies have been proposed. UltraPixel~\cite{ren2024ultrapixel} leverages cascaded diffusion models to progressively synthesize images with rich details, effectively advancing ultra-high-resolution generation.
SANA~\cite{xie2024sana} achieves efficient 4K synthesis by integrating linear transformers with deep compressed autoencoders~\cite{chen2024deep}, substantially accelerating generation while maintaining text-image alignment.
Diffusion-4K~\cite{zhang2025diffusion} introduces Wavelet-based Latent Fine-tuning (WLF), extending the capabilities of large-scale pre-trained LDMs to the 4K domain.
Furthermore, UltraImage~\cite{zhao2025ultraimage} employs recursive dominant-frequency correction to mitigate repetitive artifacts, alongside an entropy-guided adaptive attention mechanism to recover sharpness lost during resolution extrapolation, facilitating high-fidelity generation at extreme scales.

\subsection{Representation Alignment for Generative Models}

Proven effective in improving spatial structure, representation alignment injects deep semantic priors from vision foundation models into generative latent features to enhance representation learnability~\cite{yu2024representation, singh2025matters}.
Existing literature can be broadly categorized into two main trajectories: aligning intermediate diffusion features and enriching autoencoder latent spaces.

\textbf{Alignment within Diffusion Models.} 
Within the context of denoising networks, REPA~\cite{yu2024representation} pioneers the cross-model alignment of noisy intermediate hidden states with clean, robust image representations extracted from external pre-trained visual encoders.
Building upon this, iREPA~\cite{singh2025matters} highlights the critical role of spatial structures in this distillation process.
By emphasizing the transfer of spatial information, iREPA demonstrates accelerated convergence and strong adaptability across diverse architectures, including latent-space (DiT~\cite{peebles2023scalable}, SiT~\cite{ma2024sit}) and pixel-space (JiT~\cite{li2025back}) diffusion models.
Furthermore, REPA-E~\cite{leng2025repa} extends this paradigm by enabling joint tuning of both the VAE and the diffusion model under alignment constraints, yielding significant performance gains.

\textbf{Alignment within Autoencoder Latent Spaces.} 
Parallel to diffusion-centric methods, alignment strategies have been increasingly applied to the latent spaces of autoencoders~\cite{yao2025reconstruction, zheng2025diffusion, tong2026scaling}.
VA-VAE~\cite{yao2025reconstruction} explicitly aligns the high-dimensional latent space of the visual tokenizer with foundation models, effectively advancing the reconstruction-generation frontier of LDMs.
Similarly, Representation Autoencoders (RAEs)~\cite{zheng2025diffusion} introduce a novel class of autoencoders that replace the conventional VAE with a pre-trained representation encoder paired with a trainable decoder, thereby directly linking deep semantic understanding with generative modeling.
Scale-RAE~\cite{tong2026scaling} further investigates the scaling laws of RAEs for text-to-image synthesis, demonstrating substantial potential in both scalable generation and unified multi-modal modeling.

Despite these pioneering efforts, directly applying REPA approaches to large-scale pre-trained LDMs inevitably induces generation degradation, exposing a critical \emph{learnability-fidelity conflict} that remains largely underexplored in ultra-high-resolution image synthesis.
Consequently, how to effectively synergize the macroscopic structural advantages of deep representation priors with the native, high-frequency generative capacity of pre-trained LDMs remains a pivotal open question.

\section{Methodology}

In this section, we elaborate on the proposed Spatial Gram Alignment (SGA) framework, systematically designed to harmonize the deep representation priors of vision foundation models with the native generative capacity of large-scale pre-trained LDMs.
We unfold our methodology as follows.
First, we establish the mathematical notations and preliminaries for large-scale pre-trained LDMs in Section~\ref{subsec:preliminaries}.
Second, we present the core formulation of SGA, detailing its role as a non-invasive structural constraint in Section~\ref{subsec:sga}.
Finally, we detail the overall optimization framework in Section~\ref{subsec:overall_framework}, illustrating the seamless integration strategies within pre-trained LDM pipelines.

\subsection{Preliminaries}
\label{subsec:preliminaries}

Large-scale pre-trained LDMs~\cite{rombach2022high} typically decouple the image generation process into two distinct stages: perceptual compression with VAEs and latent generative modeling. 

Given an image $x \in \mathbb{R}^{H \times W \times 3}$, a pre-trained VAE encoder $\mathcal{E}$ first compresses the image into a dense, lower-dimensional latent representation $z = \mathcal{E}(x) \in \mathbb{R}^{h \times w \times c}$. Conversely, the decoder $\mathcal{D}$ maps the latent code back to the pixel space, yielding the reconstructed image $\hat{x} = \mathcal{D}(z)$.

Within this latent space, recent advancements such as SD3~\cite{esser2024scaling} and Flux~\cite{Flux:2024:Online} have transitioned from standard diffusion formulations~\cite{ho2020denoising, nichol2021improved} to flow matching~\cite{lipman2022flow, albergo2022building}, specifically rectified flow~\cite{liu2022flow}, to construct the optimal transport path between the noise and data distributions. 
Formally, rectified flow defines a linear interpolation path connecting a Gaussian noise variable $z_0 \sim \mathcal{N}(0, \mathbf{I})$ and a target data variable $z_1$ (\emph{i.e.}, the encoded latent $z_1 = z \sim q(z)$). 
At any time $t \in [0, 1]$, the intermediate state is constructed as: 
\begin{equation}
    z_t = t z_1 + (1 - t) z_0.
    \label{equation: interpolate}
\end{equation}

The generative process is governed by an Ordinary Differential Equation (ODE), $dz_t = v(z_t, t) dt$, where the target vector field (\emph{i.e.}, velocity) is defined as the constant linear trajectory $v(z_t, t) = z_1 - z_0$. 
A neural network $v_\theta$, typically parameterized by a Multimodal Diffusion Transformer (MMDiT) architecture~\cite{esser2024scaling}, is trained to predict this vector field conditioned on external signals $c$ (\emph{e.g.}, text embeddings). 
The model is optimized via conditional vector field regression to minimize the flow matching objective:
\begin{equation}
    \mathcal{L}_{fm} = \mathbb{E}_{z_1 \sim q(z), z_0 \sim \mathcal{N}(0,\mathbf{I}), c, t \sim \mathcal{U}(0,1)} \left[ \| v_\theta(z_t, t, c) - (z_1 - z_0) \|_2^2 \right].
    \label{equation:flow_matching}
\end{equation}

\subsection{Spatial Gram Alignment}
\label{subsec:sga}

To effectively inject representation priors into pre-trained LDMs without compromising the pixel-level fidelity, we propose Spatial Gram Alignment (SGA).
Formally, let $f(\cdot)$ denote the frozen vision foundation model and $g(\cdot)$ denote the target generative module (\emph{e.g.}, the VAE encoder or the diffusion network). 
Given a clean image $x$, the foundation model extracts deep semantic representations directly from the pixels, yielding $H_f = f(x) \in \mathbb{R}^{N \times C_f}$. 
Concurrently, $g(\cdot)$ produces the corresponding generative feature maps $H_g \in \mathbb{R}^{N \times C_g}$. Depending on the specific alignment target within the text-to-image synthesis pipeline, $H_g$ represents either the encoded VAE latent~\cite{yao2025reconstruction} or the intermediate hidden states extracted from the denoising network processing the noisy latent $z_t$~\cite{yu2024representation}, conditioned on timestep $t$ and text prompt $c$. 
Here, $C_f$ and $C_g$ denote their respective channel dimensions, while $N = h \times w$ represents a shared macroscopic sequence length. 
In practice, to reconcile any inherent spatial mismatch, feature maps are explicitly downsampled via adaptive average pooling to this unified length $N$ prior to further alignment.

Specifically, we first map the generative features into a shared feature space via a projection head $\phi(\cdot)$. 
We then apply $L_2$ normalization along the channel axis for both the projected generative features and the foundation priors, yielding $\tilde{H}_g = \text{Norm}(\phi(H_g)) \in \mathbb{R}^{N \times C_f}$ and $\tilde{H}_f = \text{Norm}(H_f) \in \mathbb{R}^{N \times C_f}$. 
Subsequently, we construct the spatial Gram matrices $G_g, G_f \in \mathbb{R}^{N \times N}$ via straightforward matrix multiplication, which encapsulates the dense structural correlations among all spatial patches:
\begin{equation}
  G_g = \tilde{H}_g \tilde{H}_g^\top, \quad G_f = \tilde{H}_f \tilde{H}_f^\top.
  \label{equation:gram_computation}
\end{equation}

Crucially, aligning this spatial Gram matrix rather than the absolute features imposes a relative constraint over feature topology, while remaining invariant to orthogonal transformations of the projected generative feature basis (see Appendix~\ref{appendix:theory}).
Provided the inherent structure of self-similarities remains consistent, the projected generative representations retain substantially more degrees of freedom than under direct patch matching, effectively bypassing disruptive cross-model feature-coordinate homogenization.
Formally, we instantiate this non-invasive topological constraint by penalizing the structural divergence between the two spatial Gram matrices. 
This objective is directly optimized by minimizing their scaled squared Frobenius norm, which explicitly aligns the dense, pair-wise representation topologies:
\begin{equation}
  \mathcal{L}_{sga} 
  = \mathbb{E}_{x, z_0, c, t} \left[ \frac{1}{N^2} \| G_g - G_f \|_F^2 \right] 
  = \mathbb{E}_{x, z_0, c, t} \left[ \frac{1}{N^2} \sum_{i=1}^{N} \sum_{j=1}^{N} \Big( (G_g)_{i,j} - (G_f)_{i,j} \Big)^2 \right].
  \label{equation:sga_loss}
\end{equation}
In essence, rather than forcing the generative network to rigidly mimic the foundation model's feature space, Equation~\ref{equation:sga_loss} acts as a spatial structural distillation objective.
While recent works have explored empirical heuristics in orthogonal domains, such as intra-model anchoring for SSL pre-training~\cite{simeoni2025dinov3} or relational alignment in video fine-tuning~\cite{zhang2025videorepa}, these works neither identify the critical \emph{learnability-fidelity conflict} nor provide a formal account of how relational matching avoids direct feature-coordinate matching within pre-trained LDMs.
In contrast, Appendix~\ref{appendix:theory} establishes the three key properties of $\mathcal{L}_{sga}$: channel-orthogonal gauge invariance, spectral and spatial subspace matching, and zero-loss containment, which provide the theoretical basis for our precise non-invasive claim and justify the unified two-stage integration presented in Section~\ref{subsec:overall_framework}.

\subsection{Optimization Framework}
\label{subsec:overall_framework}

We integrate SGA into the LDM pipeline through a unified two-component template:
\begin{equation}
    \mathcal{L}_{\text{stage}} \;=\; \underbrace{\mathcal{L}_{\text{native}}}_{\text{stage-specific native objective}} \;+\; \lambda_s \cdot \underbrace{\mathcal{L}_{sga}\big(\phi_{\text{stage}}(H_g),\, f(x)\big)}_{\text{injects foundation prior}},
    \label{equation:unified}
\end{equation}
where $\mathcal{L}_{\text{native}}$ is the stage's native training objective and $\mathcal{L}_{sga}$ injects macroscopic structural priors from the foundation model. 
The two stages of the LDM pipeline, including VAE perceptual compression and latent generative modeling, are obtained as concrete instantiations of Eq.~\ref{equation:unified}.

\textbf{VAE stage.} Here $\mathcal{L}_{\text{native}} = \mathcal{L}_{\text{vanilla}} + \lambda_m \mathcal{L}_m$ combines the standard VAE objective with our proposed moment alignment loss anchoring the encoder to the pre-trained latent statistics, resulting in:
\begin{equation}
    \mathcal{L}_{vae} = \mathcal{L}_{vanilla}(x, \hat{x}) + \lambda_{m} \cdot \mathcal{L}_{m}(\mathcal{E}(x), \mathcal{E}^\ast(x)) + \lambda_{s} \cdot \frac{ \| \nabla_{\mathcal{E}^{L_{\mathcal{E}}}}[\mathcal{L}_{m}] \|_2 }{ \| \nabla_{\mathcal{E}^{L_{\mathcal{E}}}}[\mathcal{L}_{sga}] \|_2} \mathcal{L}_{sga}(\phi_{vae}(z), f(x)), 
    \label{equation:vae}
\end{equation}
where $\nabla_{\mathcal{E}^{L_{\mathcal{E}}}}[\cdot]$ denotes the clamped gradient of the respective loss term \emph{w.r.t.} the last layer $L_{\mathcal{E}}$ of the encoder $\mathcal{E}$~\cite{esser2021taming}.
Furthermore, $\mathcal{L}_{vanilla}$ represents the standard composite objective crucial for high-fidelity optimization~\cite{esser2021taming, rombach2022high}, comprising a pixel-wise reconstruction loss $\mathcal{L}_{rec}$, a perceptual penalty $\mathcal{L}_{lpips}$~\cite{zhang2018unreasonable}, and a patch-based adversarial loss $\mathcal{L}_{adv}$ parameterized by a discriminator $\psi_{adv}$~\cite{isola2017image}:
\begin{equation}
\begin{aligned}
    \mathcal{L}_{vanilla} = \mathop{\min}\limits_{\mathcal{D}, \mathcal{E}} \mathop{\max}\limits_{\psi_{adv}} \bigg[  \mathcal{L}_{rec}(x, \mathcal{D}(\mathcal{E}(x))) + \lambda_{lpips} \cdot \mathcal{L}_{lpips}(x, \mathcal{D}(\mathcal{E}(x))) & \\
    - \lambda_{adv} \cdot \frac{\| \nabla_{\mathcal{D}^{L_{\mathcal{D}}}}[\mathcal{L}_{lpips}] \|_2 }{ \| \nabla_{\mathcal{D}^{L_{\mathcal{D}}}}[\mathcal{L}_{adv}] \|_2 } \mathcal{L}_{adv}(x, \psi_{adv}(\mathcal{D}(\mathcal{E}(x))))   \bigg],
    \label{equation:vanilla_vae}
\end{aligned}
\end{equation}
where $\nabla_{\mathcal{D}^{L_{\mathcal{D}}}}[\cdot]$ similarly denotes the value-clamped gradient \emph{w.r.t.} the last layer ${L_{\mathcal{D}}}$ of the decoder $\mathcal{D}$. 
The moment alignment term $\mathcal{L}_{m} = \mathbb{E}_{x} \left[ \| \mu(x) - \mu^\ast(x) \|^2_2 + \| \log\sigma^2(x) - \log\sigma^{\ast 2}(x) \|^2_2 \right]$ anchors the latent mean $\mu$ and log-variance $\log\sigma^2$ to the pre-trained latent space, maintaining compatibility with the frozen LDM.
Coupled with scale consistency regularization~\cite{zhang2025both}, this alignment facilitates deeper compression while preserving the inherent statistical characteristics of the original latent manifold, thereby maintaining compatibility with the pre-trained LDM.

\textbf{Diffusion stage.} Here, $\mathcal{L}_{\text{native}}$ serves as the standard conditional diffusion objective (\emph{e.g.}, flow matching~\cite{esser2024scaling, liu2022flow} or WLF~\cite{zhang2025diffusion}). 
Because the preceding VAE stage explicitly anchors the representations within the pre-trained latent space, the generative trajectory of the LDM is implicitly preserved by initializing from the pre-trained diffusion weights $\theta^\ast$, where $\mathcal{L}_{fm}$ is already near-optimal. 
Following the established representation alignment paradigm~\cite{yu2024representation, singh2025matters}, the training objective is formulated as a direct instantiation of Eq.~\ref{equation:unified} on the intermediate hidden states $H_g$ of the denoising network:
\begin{equation}
    \mathcal{L}_{diff} = \mathcal{L}_{fm}(\theta) + \lambda_{s} \cdot \mathcal{L}_{sga}(\phi_{diff}(H_g), f(x)).
    \label{equation:diffusion}
\end{equation}
By integrating the proposed $\mathcal{L}_{sga}$, our framework seamlessly injects deep structural priors from the foundation model while effectively circumventing the generation degradation commonly induced by direct alignment paradigms. Consequently, the diffusion network is empowered to synthesize coherent global compositions without sacrificing its inherent high-frequency generative fidelity.

The complete two-stage training procedure is summarized in Algorithm~\ref{alg:sga_training}.

\begin{algorithm}[tb]
\renewcommand{\algorithmicrequire}{\textbf{Input:}} 
\renewcommand{\algorithmicensure}{\textbf{Output:}}
\caption{Training Framework of Spatial Gram Alignment (SGA).}
\label{alg:sga_training}
\begin{algorithmic}[1]
\REQUIRE 
Training dataset $\mathcal{X}$; vision foundation model $f$; patch discriminator $\psi_{adv}$; pre-trained VAE $\{\mathcal{E}^\ast, \mathcal{D}^\ast\}$ and diffusion network $\theta^\ast$; total iterations $T_{vae}$ and $T_{diff}$; loss weights $\lambda_m$ and $\lambda_{s}$. 
\ENSURE
The fine-tuned VAE $\{\mathcal{E}, \mathcal{D}\}$, the diffusion network $\theta$, and the spatial projectors $\{\phi_{vae}, \phi_{diff}\}$.

\STATE \textcolor{gray}{// \textit{Stage 1: VAE Fine-tuning for Latent Compression}}
\STATE Initialize VAE $\{\mathcal{E}, \mathcal{D}\}$ with pre-trained $\{\mathcal{E}^\ast, \mathcal{D}^\ast\}$, patch discriminator $\psi_{adv}$, and projector $\phi_{vae}$. Freeze foundation model $f$.
\FOR{$i = 1, 2, ..., T_{vae}$} 
    \STATE Sample a batch of high-resolution images $x \sim \mathcal{X}$.
    \STATE Extract representation priors $H_f = f(x)$ and reference latent statistics $z^\ast = \mathcal{E}^\ast(x)$.
    \STATE Compute VAE latents $z = \mathcal{E}(x)$, reconstructions $\hat{x} = \mathcal{D}(z)$, and adversarial objective $\mathcal{L}_{adv}$.
    \STATE $\psi_{adv} \gets \mathbf{MODELUPDATE} (\psi_{adv}, \nabla \mathcal{L}_{adv})$. \hfill 
    \STATE $\mathcal{L}_{vae} \gets \mathcal{L}_{vanilla} + \lambda_{m} \cdot \mathcal{L}_{m} + \lambda_{s} \cdot \frac{ \| \nabla_{\mathcal{E}^{L_{\mathcal{E}}}}[\mathcal{L}_{m}] \|_2 }{ \| \nabla_{\mathcal{E}^{L_{\mathcal{E}}}}[\mathcal{L}_{sga}] \|_2} \mathcal{L}_{sga}$ \hfill (Eq.~\ref{equation:vae}).
    \STATE $\{\mathcal{E}, \mathcal{D}, \phi_{vae}\} \gets \mathbf{MODELUPDATE} (\{\mathcal{E}, \mathcal{D}, \phi_{vae}\}, \nabla \mathcal{L}_{vae})$. 
\ENDFOR

\STATE \textcolor{gray}{// \textit{Stage 2: Generative Denoising Stage}}
\STATE Freeze the fine-tuned VAE $\{\mathcal{E}, \mathcal{D}\}$. 
\STATE Initialize diffusion network $\theta$ with pre-trained $\theta^\ast$, and spatial projector $\phi_{diff}$.
\FOR{$i = 1, 2, ..., T_{diff}$}
    \STATE Sample images $x \sim \mathcal{X}$ and the corresponding text prompts $c$.
    \STATE Extract representation priors $H_f = f(x)$ and compress latents $z = \mathcal{E}(x)$.
    \STATE Sample timestep $t$, compute noisy latent $z_t$, and obtain intermediate states $H_g$ from $\theta$.
    \STATE $\mathcal{L}_{diff} \gets \mathcal{L}_{fm} + \lambda_{s} \cdot \mathcal{L}_{sga}$ \hfill (Eq.~\ref{equation:diffusion}).
    \STATE $\{\theta, \phi_{diff}\} \gets \mathbf{MODELUPDATE} (\{\theta, \phi_{diff}\}, \nabla \mathcal{L}_{diff})$.
\ENDFOR

\end{algorithmic}
\end{algorithm}

\section{Experiments}

In this section, we present the implementation details of our algorithm and evaluate its performance through both quantitative metrics and qualitative visual comparisons.
Our results demonstrate the effectiveness of integrating our proposed approach with state-of-the-art LDMs for ultra-high-resolution text-to-image synthesis.

\subsection{Implementation Details}

We adopt Flux.1-dev~\cite{Flux:2024:Online} as our default LDM owing to its strong capabilities in text-to-image synthesis.
To demonstrate the generalizability of our framework, we employ SAM2-B/32~\cite{ravi2024sam}, DINOv2-B/14~\cite{oquab2023dinov2}, and DINOv3-B/16~\cite{simeoni2025dinov3} as the vision foundation models for extracting spatial representation priors.
In practice, a convolutional layer followed by an adaptive average pooling is utilized as the projection head $\phi(\cdot)$ to map the generative features into the shared feature space. 

For the optimization of the pre-trained VAE, we curate a massive high-resolution dataset comprising over 12 million images sourced from SA-1B~\cite{kirillov2023segment}, FFHQ~\cite{karras2019style}, and Mapillary Vistas~\cite{neuhold2017mapillary}, \emph{etc}.
We fine-tune the Flux VAE at a $1024 \times 1024$ resolution for 2 epochs on 16 NVIDIA H100 GPUs, with a total batch size of 160.
The loss weights $\lambda_m$, $\lambda_s$, $\lambda_{lpips}$ and $\lambda_{adv}$ are empirically set to $1.0$, $1.0$, $0.1$ and $0.05$, respectively.
The model is optimized using the AdamW optimizer~\cite{loshchilov2017decoupled} with a learning rate of $1e-5$ and a weight decay of $1e-4$.
Crucially, to directly accommodate ultra-high-resolution synthesis, we explicitly incorporate scale consistency regularization~\cite{zhang2025both} to push the VAE towards a deeper $16 \times$ spatial compression rate.

During the latent generative stage, we fine-tune the 12B Flux diffusion network on the Aesthetic-Train~\cite{zhang2025diffusion}, which consists of 12,015 carefully curated high-quality training images.
The training spans 20K iterations on 8 NVIDIA H100 GPUs with a total batch size of 32, heavily accelerated by DeepSpeed ZeRO~\cite{rajbhandari2020zero}.
To effectively synthesize 4K images, we incorporate wavelet-based latent fine-tuning~\cite{zhang2025diffusion} to preserve intricate high-frequency details. 
We couple this technique with aspect-ratio bucket training up to a $4096$ long-edge resolution. 
This naturally bypasses destructive center-cropping, thereby preserving the intrinsic visual characteristics of the original images, while simultaneously employing a logit-normal timestep sampler~\cite{esser2024scaling}.
The AdamW optimizer is employed with a learning rate of $1e-6$ and a weight decay of $1e-4$.
During the inference phase, all images are generated utilizing an Euler solver with 50 sampling steps, and the guidance scale is set to 7.0 and 5.0 for 2K and 4K image generation, respectively.
More details are provided in the Appendix.

\begin{table}
  \caption{Quantitative comparisons on Aesthetic-Eval benchmark, including both Aesthetic-Eval@2K and Aesthetic-Eval@4K at 2K and 4K scales respectively.}
  \label{table:main_results}
  \centering
  \resizebox{\columnwidth}{!}{
  \begin{tabular}{lcccccc}
    \toprule
    \multirow{2}{*}{Model} & \multirow{2}{*}{Evaluation Dataset} & \multicolumn{3}{c}{Holistic Measures} & \multicolumn{2}{c}{Local Measures} \\
    \cmidrule(lr){3-5} \cmidrule(lr){6-7}
    & & gFID $\downarrow$ & CLIP Score $\uparrow$ & Aesthetics $\uparrow$ & GLCM Score $\uparrow$ & Compression Ratio $\downarrow$ \\
    \midrule
    Flux~\cite{Flux:2024:Online} & \multirow{5}{*}{Aesthetic-Eval@2K}  & 50.57 & 30.41 & 6.36 & 0.58 & 14.80  \\
    Flux-WLF (Diffusion-4K)~\cite{zhang2025diffusion} &  & 39.49 & 34.41 & 6.37 & 0.61 & 13.60  \\
    Flux-SGA-SAM2 (Ours) &  & 38.57 & 34.46 & 6.38 & 0.79 & 9.98   \\
    Flux-SGA-DINOv2 (Ours) &   & 39.75 & 34.47 & 6.38 & 0.81 & 10.06 \\
    Flux-SGA-DINOv3 (Ours) &   & 39.39 & 34.53 & 6.39 & 0.79 & 10.49 \\
    \midrule
    Flux~\cite{Flux:2024:Online} & \multirow{5}{*}{Aesthetic-Eval@4K}  & 154.96 & 30.76 & 6.02 & 0.38 & 18.83  \\
    Flux-WLF (Diffusion-4K)~\cite{zhang2025diffusion} &  & 151.95 & 33.12 & 6.08 & 0.39 & 18.69  \\
    Flux-SGA-SAM2 (Ours) &  & 148.30 & 33.46 & 6.10 & 0.40 & 16.11   \\
    Flux-SGA-DINOv2 (Ours) &   & 146.41 & 33.60 & 6.15 & 0.47 & 15.56 \\
    Flux-SGA-DINOv3 (Ours) &   & 146.33 & 33.61 & 6.17 & 0.43 & 17.33 \\
    \bottomrule
  \end{tabular}
  }
\end{table}

\subsection{Main Results}

\textbf{Quantitative Evaluation.}
Following the standard evaluation protocol established in Diffusion-4K~\cite{zhang2025diffusion}, we assess the generation quality across macroscopic holistic metrics, including gFID~\cite{heusel2017gans}, CLIP Score~\cite{hessel2021clipscore}, and Aesthetics Score~\cite{schuhmann2022laion}, as well as microscopic local metrics, specifically GLCM Score and Compression Ratio~\cite{zhang2025diffusion}.
As summarized in Table~\ref{table:main_results}, we comprehensively evaluate our framework on the Aesthetic-Eval benchmark, which comprises 2,781 and 195 images for robust testing at 2K (Aesthetic-Eval@2K) and 4K (Aesthetic-Eval@4K) resolutions, respectively.
Experimental results show that our approach achieves clear performance gains across this broad spectrum of evaluative dimensions, reconciling global semantic coherence with local high-frequency fidelity.
Furthermore, these performance gains across most evaluative metrics are consistently maintained regardless of the choice of vision foundation models including SAM2~\cite{ravi2024sam}, DINOv2~\cite{oquab2023dinov2}, and DINOv3~\cite{simeoni2025dinov3}, highlighting the robust generalizability of the proposed SGA framework.

\textbf{Qualitative Visualizations.}
Beyond quantitative metrics, we provide rigorous qualitative comparisons against the strong Diffusion-4K baseline~\cite{zhang2025diffusion} for direct 4K image synthesis, utilizing text prompts sampled from the Aesthetic-Eval@4K benchmark.
As illustrated in Figure~\ref{fig:comparisons}, our SGA framework synthesizes photorealistic 4K images with significantly improved macroscopic structural coherence (\emph{e.g.}, accurate object proportions and logical global layouts) while simultaneously enhancing intricate, high-frequency local details (\emph{e.g.}, high-fidelity surface patterns and sharp edges).
This visual evidence clearly supports our core claim: SGA effectively reconciles the \emph{learnability-fidelity conflict}, enabling superior visual quality at extreme 4K scales.
Extensive qualitative results are provided in the Appendix.

\begin{figure}
  \centering
  \includegraphics[width=\linewidth]{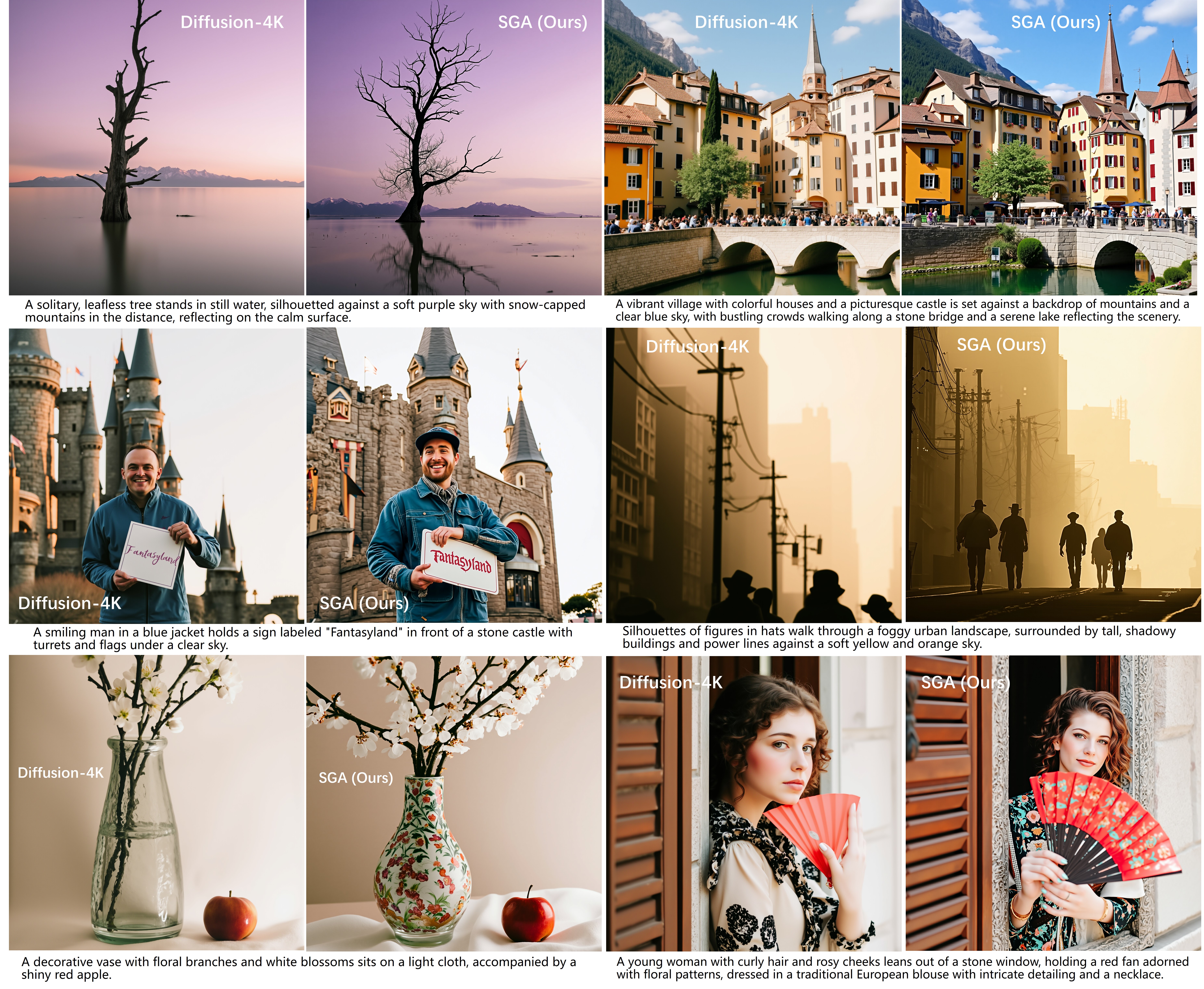}
  \caption{\textbf{Qualitative Comparisons at 4K Resolution.} Compared to the strong Diffusion-4K baseline~\cite{zhang2025diffusion}, our SGA framework achieves superior macroscopic structural coherence and effectively preserves microscopic high-frequency details, enabling superior visual quality. Please zoom in for better visualization. }
  \label{fig:comparisons}
\end{figure}

\textbf{VAE Reconstruction Performance.}
Since preserving pixel-level fidelity fundamentally begins at the latent compression stage, we further evaluate the reconstruction capability of our optimized VAE on the Aesthetic-Train dataset, which contains 12,015 high-quality 4K images and is disjoint from the large-scale corpus used for VAE training.
As depicted in Table~\ref{table:vae_results}, we present quantitative reconstruction comparisons between our optimized VAE and the partitioned off-the-shelf Flux VAE baseline~\cite{zhang2025diffusion} with $16\times$ compression ratio.
To comprehensively assess reconstruction quality, we employ a rigorous suite of metrics, including rFID, Normalized Mean Square Error (NMSE), Peak Signal-to-Noise Ratio (PSNR), Structural Similarity Index Measure (SSIM)~\cite{wang2004image}, and Learned Perceptual Image Patch Similarity (LPIPS)~\cite{zhang2018unreasonable}.
Evaluated on the Aesthetic-Train under an identical compression ratio, the results explicitly confirm that our SGA-enhanced VAE achieves superior high-fidelity reconstruction. 
More comparisons against the vanilla fine-tuning baseline~\cite{rombach2022high} are provided in Appendix~\ref{appendix:more_results}.
Notably, this comprehensive enhancement across all statistical and perceptual metrics (\emph{e.g.}, rFID and LPIPS) serves as strong evidence that our non-invasive alignment effectively preserves delicate high-frequency variations without compromising spatial integrity~\cite{bfl2025representation}. 
By safeguarding the preservation of microscopic details at the latent level, this enhanced reconstruction capability proves to be an indispensable component for high-fidelity ultra-high-resolution synthesis.

\begin{table}
  \caption{Quantitative reconstruction comparisons of VAEs with a downsampling factor of $F = 16$ on Aesthetic-4K benchmark. }
  \label{table:vae_results}
  \centering
  \resizebox{0.9\columnwidth}{!}{
  \begin{tabular}{lcccccc}
    \toprule
    Model & Resolution & rFID $\downarrow$ &  NMSE $\downarrow$ & PSNR $\uparrow$ & SSIM $\uparrow$ & LPIPS $\downarrow$ \\
    \midrule
    Flux-VAE-F16~\cite{zhang2025diffusion} & \multirow{4}{*}{$2048\times 2048$} & 1.95 & 0.10 & 27.54 & 0.77 & 0.17     \\
    Flux-VAE-F16-SGA-SAM2 (Ours)  & & 0.06 & 0.06 & 32.01 & 0.85 & 0.07   \\ 
    Flux-VAE-F16-SGA-DINOv2 (Ours) & & 0.08 & 0.07 & 31.58 & 0.84 & 0.08 \\ 
    Flux-VAE-F16-SGA-DINOv3 (Ours) & & 0.08 & 0.06 & 32.19 & 0.86 & 0.09 \\ 
    \midrule
    Flux-VAE-F16~\cite{zhang2025diffusion} & \multirow{4}{*}{$4096\times 4096$} & 1.69 & 0.08 & 29.22 & 0.79 & 0.16     \\
    Flux-VAE-F16-SGA-SAM2 (Ours) & & 0.38 & 0.06 & 33.59 & 0.85 & 0.09   \\ 
    Flux-VAE-F16-SGA-DINOv2 (Ours) & & 0.39 & 0.06 & 33.41 & 0.85 & 0.09 \\ 
    Flux-VAE-F16-SGA-DINOv3 (Ours) & & 0.44 & 0.05 & 33.99 & 0.87 & 0.09 \\ 
    \bottomrule
  \end{tabular}
  }
\end{table}

\subsection{Ablation Studies}

\textbf{Ablation on alignment strategy.}
To directly validate the central motivation of our work, we compare SGA against a representative direct patch-wise alignment method, namely iREPA~\cite{singh2025matters}.
For a controlled comparison, both methods are applied to the same visual branch of the 12-th double-stream block in Flux~\cite{Flux:2024:Online}, using the same SAM2-B~\cite{ravi2024sam} foundation prior, the same 4K fine-tuning data, and the same optimization hyperparameters for 20K training iterations.
\begin{wraptable}{r}{0.48\textwidth}
  \vspace{-10pt}
  \caption{Ablation study on alignment strategy. 
  }
  \label{table:ablation_repa}
  \centering
  \resizebox{0.45\columnwidth}{!}{
  \begin{tabular}{lcc}
    \toprule
    Method  & Alignment Loss Weight & gFID  $\downarrow$  \\
    \midrule
    Flux~\cite{Flux:2024:Online} & - & 50.57 \\
    Flux-WLF~\cite{zhang2025diffusion} & - & 39.49 \\ 
    \midrule
    \multirow{2}{*}{iREPA~\cite{singh2025matters}} & 0.1 & 55.24  \\
    & 1.0 & 274.81 \\
    \midrule
    SGA (Ours) & 1.0 & \textbf{38.57} \\ 
    \bottomrule
  \end{tabular}
  }
  \vspace{-10pt}
\end{wraptable}
As shown in Table~\ref{table:ablation_repa}, directly applying iREPA to a pre-trained LDM leads to significant generation degradation. 
Even with a conservative weight ($\lambda=0.1$), iREPA underperforms both the Flux-WLF~\cite{zhang2025diffusion} and the vanilla Flux model~\cite{Flux:2024:Online}, indicating that rigid patch-wise alignment disrupts the pre-trained generative manifold. 
This trend becomes more pronounced as the alignment weight increases, resulting in a total collapse at $\lambda=1.0$. 
In contrast, SGA remains compatible with the pre-trained manifold and outperforms both vanilla Flux and Flux-WLF under the same alignment location and training recipe, successfully harnessing the foundation priors where iREPA fails.
This failure mode of direct alignment is consistent with Proposition~\ref{prop:strict}: iREPA has a strictly tighter zero-loss set that forces the projected features off the pre-trained manifold, whereas SGA admits a strictly larger orbit of zero-loss configurations.

\textbf{Ablation on alignment with different layers.}
Furthermore, we investigate the sensitivity of SGA when aligning foundation priors with different intermediate layers of the diffusion model.
\begin{wraptable}{r}{0.46\textwidth}
  \vspace{-10pt}
  \caption{Ablation study on alignment with different intermediate layers of diffusion model. }
  \label{table:ablation_layer}
  \centering
  \resizebox{0.46\columnwidth}{!}{
  \begin{tabular}{lcc}
    \toprule
    Model & Alignment Layer Index  & gFID  $\downarrow$  \\
    \midrule
    Diffusion-4K~\cite{zhang2025diffusion} & - & 39.49 \\
    \midrule
    Flux-SGA-SAM2 & 8 & 39.23 \\
    Flux-SGA-SAM2 & 12 & \textbf{38.57} \\
    Flux-SGA-SAM2 & 19 & 38.88 \\
    \bottomrule
  \end{tabular}
  }
  \vspace{-10pt}
\end{wraptable}
Specifically, we employ SAM2-B~\cite{ravi2024sam} as the target foundation prior, applying our spatial constraint to the visual feature branches of double-stream blocks at varying architectural depths within the Flux model.
As presented in Table~\ref{table:ablation_layer}, consistent performance improvements over the strong Diffusion-4K baseline~\cite{zhang2025diffusion}, identically fine-tuned on the same data for 20K iterations, are observed regardless of the particular layer chosen for alignment. 
Notably, aligning at the middle layer (\emph{i.e.}, the 12-th layer) yields the best generative quality, effectively striking an ideal balance between macroscopic high-level semantics and microscopic local spatial details.
These findings clearly demonstrate the architectural robustness of our spatial structural constraint, proving its versatile applicability across diverse depths within the generative network.

\textbf{Ablation on SGA within VAE and diffusion model.}
To rigorously isolate the contributions of our proposed framework, we evaluate the individual and combined effects of integrating SGA into the VAE and the diffusion network.
Consistent with our previous experiments, SAM2-B~\cite{ravi2024sam} is employed as the default foundation prior.
As detailed in Table~\ref{table:ablation_sga}, independently applying SGA to either
\begin{wraptable}{r}{0.42\textwidth}
  \vspace{-10pt}
  \caption{Ablation study on SGA. }
  \label{table:ablation_sga}
  \centering
  \resizebox{0.4\columnwidth}{!}{
  \begin{tabular}{lcc}
    \toprule
    SGA in VAE & SGA in Diffusion & gFID  $\downarrow$  \\
    \midrule
    - & - & 39.49 \\
    $\checkmark$ & - & 39.35 \\
    - & $\checkmark$ & 38.82 \\
    $\checkmark$ & $\checkmark$ & \textbf{38.57} \\
    \bottomrule
  \end{tabular}
  }
  \vspace{-10pt}
\end{wraptable}
the latent compression or the generative denoising stage yields consistent improvements over the baseline. 
Crucially, their joint integration yields the best performance. 
This synergistic enhancement explicitly validates our unified alignment strategy, suggesting that macroscopic representation learnability and microscopic pixel-level fidelity can be better balanced through joint guidance across the text-to-image synthesis pipeline.

\section{Conclusion}

In this paper, we propose SGA, a novel and non-invasive representation alignment framework, addressing the fundamental \emph{learnability-fidelity conflict} inherent in fine-tuning LDMs for ultra-high-resolution image synthesis. 
By aligning internal spatial self-similarities rather than enforcing absolute feature homogenization, SGA preserves the native generative manifold while injecting representation priors from foundation models.
Extensive experiments validate the versatility of our approach, seamlessly integrating into both the VAE latent compression and the generative denoising stages. 

While this study establishes the efficacy of SGA using the state-of-the-art Flux as a representative case, exploring its generalizability across a broader spectrum of pre-trained LDM manifolds remains a promising direction. 
In the future, we aim to extend SGA towards a unified generation-understanding framework, fostering bidirectional synergy between semantic priors and generative fidelity.

{
\small

\bibliographystyle{unsrt}
\bibliography{ref.bib}

}


\newpage
\appendix

\section{Theoretical Analysis}
\label{appendix:theory}

This section establishes three properties of $\mathcal{L}_{sga}$ that clarify, respectively, what the constraint leaves free in the projected generative features, what it transfers from the foundation prior, and how it relates to REPA-style patch matching~\cite{yu2024representation,singh2025matters,leng2025repa}. Prior works that use Gram-style objectives in adjacent settings~\cite{simeoni2025dinov3, zhang2025videorepa} adopt them largely as empirical heuristics and do not provide a formal account of how such losses act on feature geometry during generative fine-tuning. 
The analysis below provides this account in the spatial $N \times N$ form used by SGA: the gauge-invariance and spectral/subspace-matching properties hold for any Frobenius-Gram relational objective and thus also apply to the prior heuristics as instances; the third property is specific to the comparison with patch-wise REPA and is what motivates our choice of the non-invasive relational form for LDM fine-tuning at 4K. 
Throughout this section, $\tilde{H}_g, \tilde{H}_f \in \mathbb{R}^{N \times C_f}$ denote the row-wise $L_2$-normalized feature matrices defined in Section~\ref{subsec:sga}, and $G_g, G_f \in \mathbb{R}^{N \times N}$ are the spatial Gram matrices of Eq.~\ref{equation:gram_computation}. For clarity, the propositions use the per-sample alignment loss
$$\ell_{sga}(\tilde{H}_g,\tilde{H}_f) := \frac{1}{N^2}\|G_g-G_f\|_F^2,$$
whose expectation over training variables gives $\mathcal{L}_{sga}$. Thus, zero-loss statements below are per-sample statements; zero expected loss means the corresponding condition holds almost surely. Under finite alignment weight $\lambda_s$, training optimizes a combined objective rather than either alignment term alone; therefore, the propositions describe the zero-loss sets and local geometric constraints that shape this objective rather than the precise endpoint of training.

\begin{proposition}[Channel-orthogonal gauge invariance]
\label{prop:gauge}
For any orthogonal matrix $Q \in O(C_f)$,
$$\ell_{sga}(\tilde{H}_g Q,\, \tilde{H}_f) \;=\; \ell_{sga}(\tilde{H}_g,\, \tilde{H}_f),$$
and $\tilde{H}_g Q$ remains row-wise $L_2$-normalized.
\end{proposition}
\begin{proof}
Each row satisfies $\|(\tilde{H}_g Q)_{i,:}\|_2 = \|(\tilde{H}_g)_{i,:}\,Q\|_2 = \|(\tilde{H}_g)_{i,:}\|_2 = 1$ since $Q$ is orthogonal, so the normalization constraint is preserved. Moreover, $(\tilde{H}_g Q)(\tilde{H}_g Q)^\top = \tilde{H}_g\, Q Q^\top\, \tilde{H}_g^\top = \tilde{H}_g \tilde{H}_g^\top = G_g$. The Frobenius distance in Eq.~\ref{equation:sga_loss} therefore depends on $\tilde{H}_g$ only through $G_g$.
\end{proof}

Since $\ell_{sga}$ depends on $\tilde{H}_g$ only through $G_g$, its zero-loss set contains an entire $O(C_f)$-orbit through $\tilde{H}_f$ in the projected space (of dimension up to $C_f(C_f-1)/2$, attained when $\tilde{H}_f$ has full column rank). This orbit freedom can be absorbed by the auxiliary projection head $\phi$, so SGA does not force the projected generative features to adopt the absolute channel coordinate system of $H_f$. Thus, SGA is non-invasive with respect to the projected channel basis: it constrains the spatial self-similarity structure of the projected features (the $N \times N$ Gram $G_g$) without imposing point-wise coordinate equality.

\begin{proposition}[Spectral and spatial subspace matching]
\label{prop:spectral}
Let $\sigma_i(\tilde{H}_g)$ denote the singular values of $\tilde{H}_g$ in non-increasing order, padded with zeros to length $N$, and let $U_g, U_f \in \mathbb{R}^{N \times k}$ collect orthonormal bases for the top-$k$ eigenspaces of $G_g, G_f$ respectively, where $1 \leq k < N$. Then
\begin{enumerate}
\item[(i)] $\sigma_i(\tilde{H}_g)^2 = \lambda_i(G_g)$ for all $i$, and the spectral mismatch is bounded by $\sum_i (\lambda_i(G_g) - \lambda_i(G_f))^2 \leq \|G_g - G_f\|_F^2$;
\item[(ii)] if the top-$k$ eigengap $\delta_k := \lambda_k(G_f) - \lambda_{k+1}(G_f) > 0$, then
$$\|\sin\Theta(U_g, U_f)\|_F \;\leq\; \frac{2\,\|G_g - G_f\|_F}{\delta_k},$$
where $\Theta$ denotes the principal angles between the corresponding subspaces.
\end{enumerate}
\end{proposition}
\begin{proof}
For (i), if $\tilde{H}_g = U \Sigma V^\top$ is the SVD, then $G_g = U \Sigma\Sigma^\top U^\top$, where $\Sigma\Sigma^\top \in \mathbb{R}^{N \times N}$ is diagonal, so the eigenvalues of $G_g$ (sorted in non-increasing order) equal the squared singular values of $\tilde{H}_g$, and the eigenvectors of $G_g$ equal the left singular vectors $U$.
The Hoffman--Wielandt inequality applied to the symmetric (hence normal) matrices $G_g, G_f$ yields the spectral bound under the same-order pairing of eigenvalues. 
For (ii), $G_g, G_f$ are symmetric, so the Davis--Kahan $\sin\Theta$ theorem~\cite{yu2015useful} applies directly to the top-$k$ principal subspace of $G_f$: only the lower gap $\delta_k = \lambda_k(G_f) - \lambda_{k+1}(G_f)$ enters, and the stated bound follows with constant $2$.
\end{proof}

Part (i) constrains the relative importance of principal spatial modes; part (ii) anchors the corresponding spatial subspace whenever the foundation prior exhibits a non-degenerate top-$k$ eigengap. As the Frobenius distance is driven toward zero, $G_g \to G_f$ pins the spectrum exactly and identifies the same invariant spatial subspaces, up to the usual sign and repeated-eigenvalue ambiguities. The remaining zero-loss freedom characterized by Proposition~\ref{prop:gauge} is therefore a channel-space freedom in the projected features, not an unconstrained change in the matched spatial self-similarity structure. When $G_f$ has near-degenerate top-$k$ eigenvalues, $\delta_k \to 0$ and the subspace bound becomes vacuous; the eigenvalues are still pinned by (i), but the corresponding subspace is identifiable only up to within-block rotation. Together, Propositions~\ref{prop:gauge} and~\ref{prop:spectral} partition the constraint cleanly: the spatial self-similarity structure (eigenvalues and identifiable top-$k$ subspaces of $G_f$) is transferred, while the channel basis of the projected features is left free.

\begin{proposition}[Containment of zero-loss sets]
\label{prop:strict}
Let $\ell_{repa}(\tilde{H}_g,\tilde{H}_f) := \frac{1}{N}\|\tilde{H}_g - \tilde{H}_f\|_F^2 = 2 - \frac{2}{N}\,\mathrm{tr}(\tilde{H}_g\tilde{H}_f^\top)$ denote the per-sample patch-wise REPA loss in squared-distance form. This differs from the canonical patch-wise cosine REPA loss only by an additive constant and a positive scale (since $\|a-b\|_2^2 = 2 - 2\langle a,b\rangle$ for unit vectors), so the gradient direction and zero-loss set are identical; we adopt this form for analytical convenience. Then for any row-normalized $\tilde{H}_g, \tilde{H}_f \in \mathbb{R}^{N \times C_f}$,
$$\ell_{sga}(\tilde{H}_g,\tilde{H}_f) \;=\; \tfrac{1}{N^2}\|G_g - G_f\|_F^2 \;\leq\; \tfrac{4}{N}\,\|\tilde{H}_g - \tilde{H}_f\|_F^2 \;=\; 4\,\ell_{repa}(\tilde{H}_g,\tilde{H}_f).$$
Moreover, the zero-loss sets satisfy
$$\{\tilde{H}_g : \ell_{repa}(\tilde{H}_g,\tilde{H}_f) = 0\} \;=\; \{\tilde{H}_f\} \;\subsetneq\; \{\tilde{H}_f Q : Q \in O(C_f)\} \;=\; \{\tilde{H}_g : \ell_{sga}(\tilde{H}_g,\tilde{H}_f) = 0\}.$$
\end{proposition}
\begin{proof}
\textbf{Bound.} Decompose $G_g - G_f = \tilde{H}_g(\tilde{H}_g - \tilde{H}_f)^\top + (\tilde{H}_g - \tilde{H}_f)\tilde{H}_f^\top$. Apply the triangle inequality and the submultiplicative bound $\|AB\|_F \leq \|A\|_F \|B\|_F$, using $\|\tilde{H}_g\|_F = \|\tilde{H}_f\|_F = \sqrt{N}$ from row-wise normalization:
$$\|G_g - G_f\|_F \;\leq\; (\|\tilde{H}_g\|_F + \|\tilde{H}_f\|_F)\,\|\tilde{H}_g - \tilde{H}_f\|_F \;=\; 2\sqrt{N}\,\|\tilde{H}_g - \tilde{H}_f\|_F.$$
Squaring and dividing by $N^2$ gives the stated inequality.

\textbf{Orbit characterization.} The forward inclusion follows from Proposition~\ref{prop:gauge}: $(\tilde{H}_f Q)(\tilde{H}_f Q)^\top = \tilde{H}_f \tilde{H}_f^\top = G_f$ for any $Q \in O(C_f)$. For the reverse, suppose $G_g = G_f$. Let $r=\mathrm{rank}(G_f)$ and choose compact SVDs
$$\tilde{H}_f = U_r \Sigma_r V_f^\top, \qquad \tilde{H}_g = U_r \Sigma_r V_g^\top,$$
where the same $U_r$ and $\Sigma_r$ can be chosen because $\tilde{H}_g\tilde{H}_g^\top=\tilde{H}_f\tilde{H}_f^\top$. Complete $V_f,V_g\in\mathbb{R}^{C_f\times r}$ to orthonormal bases $[V_f,V_f^\perp]$ and $[V_g,V_g^\perp]$ of $\mathbb{R}^{C_f}$, and define
$$Q = V_f V_g^\top + V_f^\perp (V_g^\perp)^\top \in O(C_f).$$
Then $\tilde{H}_f Q = U_r\Sigma_r V_g^\top = \tilde{H}_g$. If $r<C_f$, $Q$ is not unique because rotations on the orthogonal complement of $\mathrm{row}(\tilde{H}_f)$ do not affect $\tilde{H}_fQ$; the distinct zero-loss configurations form the orbit, which corresponds to the group $O(C_f)$ modulo this stabilizer. 
Strictness is immediate because row-normalization implies $\tilde{H}_f\neq 0$, and choosing $Q=-I$ gives $\tilde{H}_fQ\neq\tilde{H}_f$ with zero SGA loss, whereas the REPA zero-loss set is the singleton $\{\tilde{H}_f\}$.
\end{proof}

The constraint imposed by SGA is therefore strictly weaker than REPA in two complementary ways: pointwise, $\ell_{sga} \leq 4 \, \ell_{repa}$ for every input pair, and in the zero-loss set, where SGA admits the entire $O(C_f)$-orbit through $\tilde{H}_f$ while REPA admits only the singleton $\{\tilde{H}_f\}$. The bound is one-directional by design and worst-case: any $\tilde{H}_g$ in the orbit $\tilde{H}_f O(C_f) \setminus \{\tilde{H}_f\}$ attains zero SGA loss while $\ell_{repa}$ can be as large as $4$ (achieved at antipodal orbit points such as $-\tilde{H}_f$). This asymmetry is precisely the additional projected-feature freedom that SGA grants and REPA forbids.

Geometrically, if $\tilde{H}_g^\ast$ denotes the projected, normalized feature induced by the pre-trained LDM, the minimum projected-feature displacement to a zero-loss configuration is $\|\tilde{H}_g^\ast - \tilde{H}_f\|_F$ for REPA but only $\min_{Q \in O(C_f)} \|\tilde{H}_g^\ast - \tilde{H}_f Q\|_F$ for SGA, the latter bounded above by the former and generically strictly smaller. This is a projected-feature loss-landscape statement, not a proof of a particular training trajectory; it is nevertheless consistent with the gFID gap between iREPA and SGA at the same alignment location and weight in Table~\ref{table:ablation_repa}. It formalizes our use of non-invasive: SGA transfers relational spatial topology without requiring the projected generative features to adopt the foundation model's absolute channel coordinates.

\section{More Details}

In this section, we provide further implementation details regarding our SGA framework. 
In practice, a convolutional layer, followed conditionally by an adaptive average pooling operation when spatial downsampling is required, is utilized as the projection head $\phi(\cdot)$ to map the generative features into the shared feature space.

For the VAE module, operating on an input image resolution of $1024 \times 1024$, the resulting highly compressed latent feature maps (with a spatial compression factor of $F=16$) inherently possess a $64 \times 64$ spatial resolution. 
When utilizing DINOv2-B/14~\cite{oquab2023dinov2} or DINOv3-B/16~\cite{simeoni2025dinov3} priors, these latents naturally align with the foundation features without requiring spatial pooling. 
Conversely, when employing SAM2-B/32~\cite{ravi2024sam}, we explicitly downsample the generative latents to $32 \times 32$ via adaptive average pooling. 
To accurately yield these target foundation feature dimensions without introducing destructive interpolation artifacts, the input image resolutions fed into the respective foundation models are dynamically resized according to their inherent patch sizes—namely, $896 \times 896$ for DINOv2-B/14, and $1024 \times 1024$ for both DINOv3-B/16 and SAM2-B/32.

Regarding the generative diffusion networks for 4K training, a stride of 2 is applied within the projection head $\phi(\cdot)$. 
Unlike the VAE stage, the diffusion models employ bucket training to accommodate diverse aspect ratios at extreme resolutions. 
Consequently, spatial alignments are dynamically scaled based on long-edge dimensions rather than fixed square grids. 
Utilizing the aforementioned adaptive average pooling, the intermediate hidden states are uniformly synchronized to a long-edge resolution of 32 across the DINOv2, DINOv3, and SAM2 priors, with the foundation models' input resolutions proportionally adapted according to their inherent patch sizes.

\section{More Results}
\label{appendix:more_results}

\begin{table}[h]
  \caption{Ablation study on SGA within VAE fine-tuning. }
  \label{table:ablation_vae}
  \centering
  \resizebox{0.8\columnwidth}{!}{
  \begin{tabular}{lccccc}
    \toprule
    Model & rFID $\downarrow$ &  NMSE $\downarrow$ & PSNR $\uparrow$ & SSIM $\uparrow$ & LPIPS $\downarrow$ \\
    \midrule
    Flux-VAE-F16-FT~\cite{rombach2022high} &  0.17 & 0.10 & 28.36 & 0.80 & 0.09   \\
    Flux-VAE-F16-SGA-SAM2 (Ours)  & 0.06 & 0.06 & 32.01 & 0.85 & 0.07   \\ 
    Flux-VAE-F16-SGA-DINOv2 (Ours) & 0.08 & 0.07 & 31.58 & 0.84 & 0.08 \\ 
    Flux-VAE-F16-SGA-DINOv3 (Ours) & 0.08 & 0.06 & 32.19 & 0.86 & 0.09 \\ 
    \bottomrule
  \end{tabular}
  }
\end{table}

In this section, we provide quantitative results comparing VAE fine-tuning with SGA against the vanilla fine-tuning baseline~\cite{rombach2022high}, which relies on standard pixel-wise and perceptual losses. 
As depicted in Table~\ref{table:ablation_vae}, the experimental results clearly demonstrate the effectiveness of fine-tuning the VAE with our proposed SGA.

\begin{table}[h]
  \caption{Impact of SGA alignment resolution on 4K synthesis performance. }
  \label{table:ablation_resolutions}
  \centering
  \resizebox{\columnwidth}{!}{
  \begin{tabular}{lcccccc}
    \toprule
    Model & Evaluation Dataset & gFID $\downarrow$ & CLIP Score $\uparrow$ & Aesthetics $\uparrow$ & GLCM Score $\uparrow$ & Compression Ratio $\downarrow$ \\
    \midrule
    Flux-SGA-DINOv2@32 &  \multirow{2}{*}{Aesthetic-Eval@4K} & 146.41 & 33.60 & 6.15 & 0.47 & 15.56 \\
    Flux-SGA-DINOv2@64 &   & 145.96 & 33.92 & 6.23 & 0.56 & 14.78 \\
    \bottomrule
  \end{tabular}
  }
\end{table}

Additionally, we investigate further optimizations tailored to maximize the generative potential of 4K image synthesis. 
As reported in Table~\ref{table:ablation_resolutions}, elevating the intermediate hidden states to a long-edge resolution of 64 with the DINOv2 prior comprehensively boosts the 4K synthesis quality across all evaluated metrics. 
This performance leap explicitly demonstrates that extreme-scale 4K synthesis fundamentally relies on higher-resolution spatial alignments within the SGA module.
By providing a denser spatial prior, our SGA acts as a dedicated 4K optimization, seamlessly unlocking the intrinsic informational capacity and fine-grained structural fidelity that ultra-high-resolution imagery demands.

Furthermore, we present additional qualitative 4K visualizations in Figure~\ref{fig:qualitative_results} to demonstrate the generative efficacy of our framework. 
These supplementary examples showcase its capacity to consistently synthesize extreme-resolution imagery, effectively balancing macroscopic structural coherence with fine-grained, high-frequency microscopic details.

\begin{figure}
  \centering
  \includegraphics[width=\linewidth]{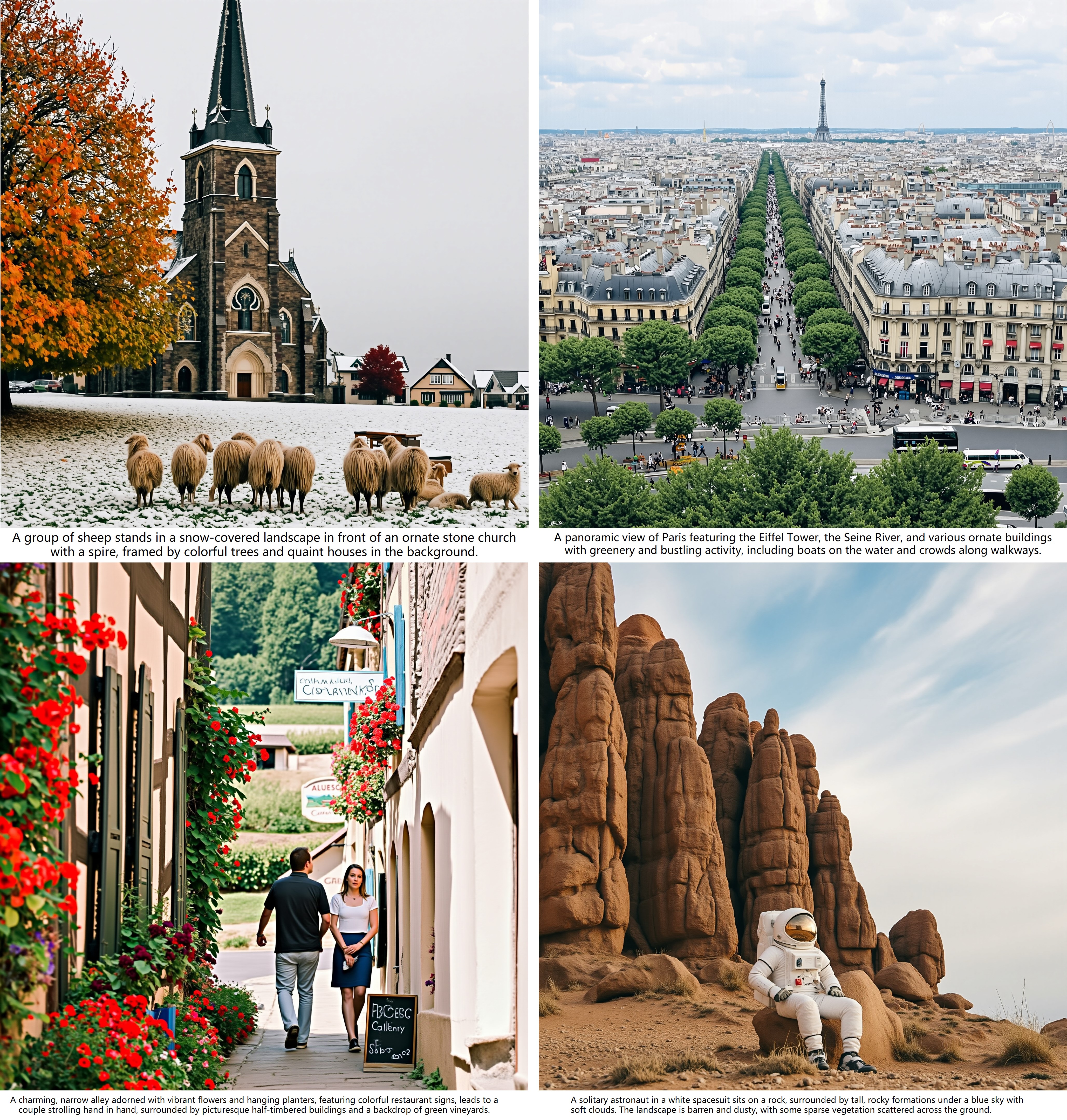}
  \caption{\textbf{Supplementary 4K Visualizations.} Additional qualitative examples generated by our SGA framework. Note the rigorous preservation of logical global topologies alongside photorealistic, crisp textures at extreme resolutions.}
  \label{fig:qualitative_results}
\end{figure}

\begin{table}
  \caption{Comparison of average time cost for the alignment loss computation. }
  \label{table:time_cost}
  \centering
  \resizebox{0.35\columnwidth}{!}{
  \begin{tabular}{lc}
    \toprule
    Method & Average Time Cost \\
    \midrule
    iREPA~\cite{singh2025matters} &  0.008 s   \\
    SGA (Ours)  & 0.006 s  \\ 
    \bottomrule
  \end{tabular}
  }
\end{table}

We compare the computational efficiency of different alignment approaches in Table~\ref{table:time_cost}. 
Benchmarks are conducted on a single NVIDIA H100 GPU with a batch size of 4 at identical resolutions. 
To isolate the specific alignment overhead, we report the average time cost for the forward pass of the alignment loss computation, excluding the duration of feature extraction. 
The results indicate that our SGA achieves superior alignment quality while maintaining a computational efficiency comparable to current patch-matching heuristics like iREPA~\cite{singh2025matters}. 
Specifically, since the alignment is performed on latent-space feature maps, the spatial dimension $N$ remains relatively modest (e.g., $32 \times 32$ or $64 \times 64$), ensuring that the $N \times N$ Gram matrix computation is highly efficient and effectively parallelized via standard GEMM kernels on modern GPUs. 
In contrast, while iREPA involves matrix operations of a smaller spatial scale, it necessitates additional per-patch normalization steps, which introduce a comparable computational overhead. 
Crucially, the time cost for both alignment methods represents a negligible fraction of the total training iteration time, ensuring that SGA provides significant generative gains with virtually no impact on overall 4K training throughput.

\section{Broader Impacts and Safeguards}

The capability to directly synthesize 4K photorealistic imagery accelerates workflows across creative industries, democratizing advanced content creation. However, we acknowledge the inherent dual-use nature of generative modeling. 
The superior structural and pixel-level fidelity of our model could be misappropriated to generate hyper-realistic misinformation, deepfakes, or unauthorized materials. 
Furthermore, we recognize the risk of systemic biases inherited from pre-training datasets, highlighting the critical need for fair demographic representation.

To responsibly mitigate these ethical and societal risks, we will adhere to a strict safeguard protocol: the pre-trained weights and inference code of our 4K framework will be released under a restricted open-access license. 
This explicitly prohibits the generation of non-consensual deepfakes, deceptive political content, and illicit materials.



\end{document}